\documentclass{article}

    \PassOptionsToPackage{numbers, compress}{natbib}

    \usepackage[preprint]{neurips_2019}

\usepackage[utf8]{inputenc} 
\usepackage[T1]{fontenc}    
\usepackage{hyperref}       
\usepackage{url}            
\usepackage{booktabs}       
\usepackage{amsfonts}       
\usepackage{nicefrac}       
\usepackage{microtype}      

\usepackage{caption}
\usepackage{subcaption}

\usepackage{amsmath}
\usepackage{amsfonts}
\usepackage{amssymb}

\newcommand{\A}{\mathcal{A}}

\newcommand{\transpose}{{\mbox{\tiny T}}}
\usepackage{graphicx}

\usepackage{xcolor}

\usepackage[ruled,vlined]{algorithm2e}
\usepackage{algorithmic}

\newcommand{\comments}[1]{{\color{blue}\textit{$\#$ #1}}}

\newtheorem{theorem}{Theorem}
\newtheorem{proposition}[theorem]{Proposition}

\newtheorem{lemma}{Lemma}
\newenvironment{proof}[1][Proof]{\noindent\textbf{#1.} }{\ \rule{0.5em}{0.5em}}

\newcommand{\cA}{{\mathcal{A}}}

\newcommand{\cD}{{\mathcal{D}}}
\newcommand{\cS}{{\mathcal{S}}}

\newcommand{\bI}{{\textbf{I}}}
\newcommand{\bJ}{{\textbf{J}}}
\newcommand{\bF}{{\textbf{F}}}
\newcommand{\bM}{{\textbf{M}}}
\newcommand{\bP}{\textbf{P}}
\newcommand{\bQ}{\textbf{Q}}
\newcommand{\bB}{\textbf{B}}
\newcommand{\bX}{\textbf{X}}
\newcommand{\bY}{\textbf{Y}}
\newcommand{\bz}{\textbf{z}}
\newcommand{\bU}{\textbf{U}}
\newcommand{\bD}{\textbf{D}}
\newcommand{\bH}{\textbf{H}}
\newcommand{\bb}{\textbf{b}}
\newcommand{\bd}{\textbf{d}}
\newcommand{\bq}{\textbf{q}}
\newcommand{\bc}{\textbf{c}}

\newcommand{\bZ}{\textbf{Z}}
\newcommand{\bPi}{\pmb{\Pi}}
\newcommand{\bpi}{\pmb{\pi}}

\newcommand{\bL}{\textbf{L}}

\title{Inverse Reinforcement Learning with Missing Data}

\author{%
 Tien Mai\\
 Singapore-MIT Alliance\\
 for Research and Technology (SMART)\\
  \texttt{mai.tien@smart.mit.edu} \\
  \And
Quoc Phong Nguyen \\
  NUS School of Computing \\
  \texttt{qphong@comp.nus.edu.sg}
    \And
Kian Hsiang Low \\
  NUS School of Computing \\
  \texttt{lowkh@comp.nus.edu.sg}
   \And
  Patrick Jaillet \\
  MIT, EECS, LIDS, ORC \\
\texttt{jaillet@mit.edu}
}

\newif\ifnotes\notestrue
\def\htien#1{}

\def\hpatrick#1{}

\begin{document}

\maketitle

\begin{abstract}
We consider the problem of  recovering an
expert’s reward function with \textit{inverse reinforcement learning (IRL)} when there are missing/incomplete state-action pairs or observations in the demonstrated trajectories. 
This issue of missing trajectory data or information occurs in many situations, e.g., GPS signals from vehicles moving on a road network are intermittent.
In this paper, we propose a tractable approach to directly compute the log-likelihood of demonstrated trajectories with incomplete/missing data. Our algorithm is efficient in handling a large number of missing segments in the demonstrated trajectories, as it  performs the training with incomplete data by solving a sequence of \textit{systems of linear equations}, and the number of such systems to be solved does not depend on the number of missing segments.
Empirical evaluation on a real-world dataset shows that our training algorithm outperforms other conventional techniques.

\end{abstract}

\section{Introduction}
The \textit{inverse reinforcement learning} (IRL) problem involves an agent using demonstrated trajectories to learn an expert's reward function \cite{Russell98}. 
IRL has been receiving a great deal of research \cite{Ng00, Ng04, Ziebart2008maximum, Levine11} due to its real-world applications.
 The rationale behind this approach is that although a reward function is a more succinct and generalizable representation of an expert's  behavior, it is often difficult for the expert to elucidate his/her reward function, as opposed to giving demonstrations. 
However, a common issue in obtaining the demonstrated trajectories is that certain state-action pairs are missing (namely, missing segments) due to, for instance, technical issues or privacy concerns. 
Some examples of this issue are the loss of GPS signal from vehicles moving on a road network, or the limit on the field of view of a  robot  when seeking to learn a patroller’s behavior in order
to penetrate the patrol without being spotted \citep{bogert2014multi}.
Missing data is also a common issue in healthcare applications \citep{Penny2012approaches}.
Hence, learning a reward function with incomplete demonstration in a scalable manner is an indispensable milestone to make IRL more practical. This is also our aim in this paper.

A straightforward approach to train IRL models with missing data is to use only the connected segments in the incomplete demonstrations \cite{bogert2014multi}. It clearly ignores the information provided by the missing segments, which can be incorporated into IRL algorithms to improve the performance.
A prominent line of work that uses the missing segment information \cite{IRL_Bogert2016, IRL_Shahryari2017,bogert2017scaling} is based on the\textit{ expectation maximization}  (EM) algorithm and the latent maximum entropy principle \cite{wang2002latent} to generalize the maximum entropy IRL \cite{Ziebart2008maximum}.
However, apart from inheriting the approximation of the latent maximum entropy principle, a more severe limitation is the expensive sampling of the missing state-action pairs in the EM algorithm, which is not effective when the length of the missing segment increases as shown in our experiments. More importantly, our experiments show that if the sampling of paths for missing segments is not sufficient, EM algorithm could be even outperformed by the above naive approach, which does not justify the complication in using missing segments.
These shortcomings of existing works and the importance of missing data leave an important open question of how to exploit accurately and efficiently the information of missing segments in IRL, which we address in this paper.

\textbf{Contribution. }
In this work, we propose a novel and efficient algorithm to train maximum entropy IRL models with full or missing data. 
Our contribution is twofold. First, we show that the training of maximum entropy IRL models with complete trajectories can be done efficiently through solving several systems of linear equations, and the number of such systems  to be solved
does not depend on the number of demonstrated trajectories. Thus, our approach is  scalable  and efficient in training IRL models with large numbers of states/actions and trajectories. Moreover, we establish conditions under which these systems of linear equations have unique solutions. We also show that, under the same conditions, the conventional \textit{value iteration} approach that has been popularly used to train maximum entropy IRL models also converges to unique fixed point solutions from any starting points.

Second, we propose a novel way to directly compute the log-likelihood of demonstrated trajectories with missing data. Based on this, we design a tractable training algorithm to deal with the issue of missing data/information. 
Our approach relies on the remark that  the probability of reaching any state from a state can be computed without sampling paths between these two states. Instead, we show that such probabilities can be obtained via solving a system of linear equations. We further propose a way to compose such systems, in such a way that we only need to solve one linear system to obtain all the probabilities of the missing segments in trajectories sharing the same set of zero-reward absorbing states.
Moreover, we show that these systems of linear equations always have unique solutions. 
The main advantage of our algorithm is that the number of such systems to be solved is independent of the number of missing segments in the demonstrated trajectories, which makes the algorithm highly scalable in large-scale settings.  

We empirically evaluate the performance of our algorithm using a dataset from real-world taxi trajectories in a large road network (thousands of links).
We show that,  with full demonstrated trajectories, our approach is able to speed up the training process up to 60 times, as compared to the widely-used \textit{value iteration} method \citep{Ziebart2008maximum}. When the demonstrated trajectories contains missing segments, we show that our algorithm outperforms two conventional approaches, i.e., the EM method and the (naive) one relying on ignoring all the missing segments in the trajectories.


The remainder of this paper is organized as follows. In Section \ref{sec:background}, we give a brief introduction to the maximum entropy IRL. In Section \ref{sec:scalable_IRL}, we present our approach to train IRL models via solving systems of linear equations, and discuss conditions under which such systems has solutions.
In Section  \ref{sec:composition_missing}, we present our main IRL training algorithm  with missing data. We evaluate and compare our approach in Section \ref{sec:experiments},  and conclude in Section \ref{sec:concl}.  

\textbf{Notation. } We use $|\cS|$ to denote the cardinality of set $\cS$. We use $\bM^\transpose$  and $\bM^{-1}$ to denote the transpose and the inverse of matrix $\bM$, respectively. We use $\bz_k$ to denote the $k$-th element of vector $\bz$, and $\bM_{h,k}$ to denote the element of matrix $\bM$ at the $h$-th row and $k$-th column.   We also use $({\partial \bM})/({\partial \theta})$ to denote an element-wise first-order derivative of $\bM$ w.r.t. parameter $\theta$, i.e., a matrix (or vector) of the same size as $\bM$ with the entry at the $h$-th row and $k$-th column being $({\partial \bM_{h,k}})/({\partial \theta})$. The Hadamard product (element-wise product) between two matrices $\bM, \bU$  is denoted by $\bM\circ\bU$. We use $||\bz||_{\infty}$ to denote the infinity norm (or maximum norm) of vector $\bz$.
\vspace{-1em}
\section{Background}\label{sec:background}
A \textit{Markov decision process} (MDP) for an agent is  defined as $(\cS,\cA,p,r,\eta)$, where $\cS$ is a set of states $\cS = \{1,2,\ldots,|\cS|\}$, $\cA$ is a finite set of actions, $p:\cS\times \cA\times\cS \rightarrow [0,1]$ is a transition probability function, i.e., $p(s'|a,s)$ is the probability of moving to state $s'\in\cS$ from $s\in \cS$ by performing action  $a\in \A$, $r(s'|s,\theta)$ is a reward function of parameters $\theta$ and a feature vector $f(s'|s)$ associated with state $s,s'\in\cS$, and $\eta$ is a discount factor.
In the context of maximum entropy inverse reinforcement learning (ME-IRL) \citep{Ziebart2008maximum} we assume that $\eta = 1$. In this paper, for notational simplicity, we assume that the reward function has a linear-in-parameters form $r(s'|s,\theta) = \theta^\transpose f(s'|s)$,
but our results can be straightforwardly applied to nonlinear reward functions (e.g., reward functions represented by artificial neutral networks). We also denote by $\cD$ the set of zero-reward absorbing states. 

According \cite{Ziebart2008maximum}, to train a ME-IRL model, we need to compute values $\bz_s$, for all states $s\in \cS$, which satisfy the following recursive equations
\begin{equation}\label{eq:recursiveIRL}
    \bz_s =
\begin{cases}    
    \sum_{a\in \cA} \sum_{s'\in \cS} p(s'|a,s) e^{r(s'|s,\theta)} \bz_{s'}   &\text{ if } s\in \cS\backslash \cD\\
    1  &\text{ if }s\in \cD.
\end{cases}
\end{equation}
Given  vector $\bz$, we can obtain the log-likelihood of demonstrated trajectories (and its gradients) by computing the local action probabilities of the following form
\begin{equation}\label{eq:action-prob}
		P(a|s) = \frac{\sum_{s'}p(s'|a,s) e^{r(s'|s,\theta)} \bz_{s'}}{\bz_s},     
\end{equation}
and the gradients w.r.t. parameters $\theta_t$
\begin{equation}\label{eq:grad-action-prob}
\begin{aligned}
		\frac{\partial P(a|s)}{\partial \theta_t} = \frac{1}{\bz_s} &\left(\sum_{s'}p(s'|a,s)e^{r(s'|s,\theta)} \left(f(s'|s)_t \bz_{s'}+ \bJ^\bz_{s',t}\right) \right)- \frac{\sum_{s'}p(s'|a,s) e^{r(s'|s,\theta)} \bz_{s'}\bJ^\bz_{s',t}}{\bz^2_s},
\end{aligned}
\end{equation}
where $\bJ^\bz$ is the Jacobian matrix of $\bz$, i.e., $\bJ^\bz =  \left[\frac{\partial \bz}{\partial \theta_1},\ldots, \frac{\partial \bz}{\partial \theta_T}\right]$.
So, to efficiently train IRL models, given any parameters $\theta$, we need to quickly compute $\bz$ and its Jacobian matrix. Traditional approaches rely on \textit{value iteration}, which could be time-consuming  and inaccurate. The method proposed in the following provides a more efficient way to obtain $\bz$ and $\bJ^{\bz}$.

\section{Scalable IRL Training Approach}\label{sec:scalable_IRL}
In this section, we show that the  vector $\bz$ as well as its Jacobian matrix can be obtained by solving  \textit{systems of linear equations}. Our approach is  scalable, in the sense that it is much faster then the conventional \textit{value iteration} \citep{Ziebart2008maximum} approach, and  would be useful to train IRL models with large numbers of states and demonstrated trajectories. We describe  our approach in the following.

Suppose that the states in $\cS$ are numbered in such a way that the absorbing states are those numbered as $|\cS| - |\cD|+1,\ldots,|\cS|$. We define a matrix $\bM$ of size ($|\cS|\times|\cS|$)  and a vector $\bb$ of size ($|\cS|$) with elements 
\[
\begin{aligned}
\bM_{s,s'} &= \left(\sum_{a\in\cA} p(s'|a,s)e^{r(s'|s,\theta)} \right), \forall s,s'\in \cS\\
\bb_s &= 1 \text{ if } s\in \cD  \text{ and 0 otherwise}.
\end{aligned}
\]
We can reformulate the recursive system  \eqref{eq:recursiveIRL} as a system of linear equations as
\begin{equation}\label{eq:Z-linear-equation}
\bz = \bM \bz +\bb \text{ or } \bz = (\bI - \bM)^{-1}\bb,    
\end{equation}
where $\bI$ is the identity matrix of size $|\cS|\times|\cS|$. So, we can obtain $\bz$ by solving a system of linear equations, which is computationally convenient.  Moreover, taking the gradient of $\bz$ w.r.t. parameter $\theta_t$ we have
\begin{equation}\label{eq:gradZ-linear-equation}
\frac{\partial \bz}{\partial \theta_t} = (\bI - \bM)^{-1} \bU^t \bz,
\end{equation}
where $\bU^t$ is a matrix of size $|\cS|\times |\cS|$ with elements
\[
\bU^t_{s,s'} = \frac{\partial \bM_{s,s'}}{\partial \theta_t} = \sum_{a\in\cA} p(s'|a,s)e^{r(s'|s,\theta)}f(s'|s)_t, 
\]
where $f(s'|s)_t$ is the $t$-th element of the vector feature $f(s'|s)$. 
If we define  $\bF^t$ as a matrix of size $|\cS|\times|\cS|$ with elements $\bF^t_{s,s'} = f(s'|s)_t$, we can write
${\partial \bz}/{\partial \theta_t} = (\bI - \bM)^{-1} (\bM\circ \bF) \bz.
$
So again,  the Jacobian of $\bz$ w.r.t. $\theta$ can be obtained conveniently by solving a  system of linear equations. 
It is possible to further speed up the computation of the Jacobian of $\bz$ by defining a matrix $\bH$ of size $|\cS|\times T$ whose  $t$-th column is $\bU^t \bz$, where $T$ is the size of $\theta$. The Jacobian matrix of $\bz$ can be computed as
$
\bJ^\bz = \left[{\partial \bz}/{\partial \theta_1},\ldots, {\partial \bz}/{\partial \theta_T}\right]=  (\bI - \bM)^{-1} \bH,
$
which is also a linear system. 
Moreover, the following theorem shows that, under some conditions,  the systems \eqref{eq:Z-linear-equation} and \eqref{eq:gradZ-linear-equation}  have unique solutions.
\vspace{-0.5em}
\begin{theorem}[Conditions for the existence and uniqueness of $\bz$ and $\bJ^{\bz}$]\label{theor:theor-1}
$\bI-\bM$ is invertible and systems \eqref{eq:Z-linear-equation} and \eqref{eq:gradZ-linear-equation} have unique solutions if one of the following conditions holds
\begin{itemize}
    \item[(i)] 
    $
    \prod_{k=1}^{K-1}\left( \sum_{a\in \cA} p(s_{k+1}|a,s_k)\right)  = 0
    $ for any sequence $\{s_1,\ldots,s_K\}$ such that $s_1=s_K$.
    \item[(ii)]  $\sum_{s'\in \cS} \bM_{s,s'}<1$ for all $s,s'\in\cS$.
\end{itemize}
\end{theorem}
We provide the proof of Theorem \ref{theor:theor-1} in the Appendix. Condition (i) holds if the network of states is cycle-free, i.e., if we leave a state $s\in\cS$, the probability of returning back to that state is zero. Condition (ii) holds generally in cases that 
  the magnitudes of parameters $\theta$ are large enough and the reward function $r$ takes significantly negative values.

An alternative approach to  compute  vector $\bz$ is to use \textit{value iteration} \citep{Ziebart2008maximum}, i.e., we perform $\bz^{k+1} \leftarrow \bM \bz^k + \bb$ for $k=0,1,\ldots$ until converging to a fixed point solution. This  approach has been used in several studies to train IRL models. However, to the best of our knowledge, there is no study  investigating conditions under which the iterative process converges
Proposition \ref{prop:convergence-value-iteration} below shows that, under the same conditions stated in Theorem \ref{theor:theor-1}, the value iteration converges to a unique fixed point solution from any starting vector $\bz^0$. The proof of Proposition \ref{prop:convergence-value-iteration} can be found in the Appendix.
\vspace{-0.5em}
\begin{proposition}[Conditions for the convergence of the value iteration]
\label{prop:convergence-value-iteration}
If one of the two conditions stated in Theorem \ref{theor:theor-1} holds, the value iteration procedure $\bz^{k+1} \leftarrow \bM \bz^k + \bb$ for $k=0,1,\ldots$ always converges to a unique fixed point solution from any starting vector $\bz^0\in\mathbb{R}^{|\cS|}_+$. Moreover,
if Condition (i) holds, then the fixed point solution $\bz^*$ lies in $[0,1]^{|\cS|}$. In addition, 
if Condition (ii) holds, then the value iteration converges after a finite number of iterations. 
\end{proposition}
\vspace{-0.5em}
\textbf{Remark.}  If both Condition (i) and (ii) hold, then the value iteration will converge to the unique fixed point solution in  $[0,1]^{|\cS|}$ after a finite number of iterations. If only Condition (i) holds, then it is not necessary that the fixed point solution lies in $[0,1]^{|\cS|}$, and if only Condition (ii) holds,  then the value iteration procedure may need an infinite number of iterations to converge to the fixed point solution. For the later, Proposition \ref{prop:estimate-no-value-iterations} gives an estimate of the number of iterations necessary to get a good approximation of the fixed point solution.
\vspace{-0.5em}
\begin{proposition}
\label{prop:estimate-no-value-iterations}
If Condition (ii) of Theorem \ref{theor:theor-1} holds, then for any $\epsilon>0$,  $ ||\bz^k-\bz^*||_\infty <\epsilon$ for all $k>\ln\epsilon/\ln\tau$, where $\tau = \max_{s\in\cS} \sum_{s'}\bM_{s,s'}$. 
\end{proposition}
\vspace{-0.5em}
\begin{proof}
From the proof of Proposition \ref{prop:convergence-value-iteration} we have inequality $||\bz^{k+1}-\bz^*||_\infty \leq \tau||\bz^k - \bz^*||_\infty$, for $k=0,1,\ldots$. This leads to $||\bz^{k}-\bz^*||_\infty\leq \tau^k$, $\forall k\in\mathbb{N}$. So, for all $k>\ln\epsilon/\ln\tau$ we have $\tau^k< \epsilon$ and  $||\bz^{k}-\bz^*||_\infty\leq \tau^k<\epsilon$, as required.
\end{proof}

 It is possible to further accelerate the computation of vectors $\bz$ and their Jacobian matrices. The idea is to compose the linear systems from different demonstrated trajectories into one linear system. First, let us assume that we need to compute the log-likelihood of $N$ demonstrated trajectories $\sigma^1,\ldots,\sigma^N$. For each observation $\sigma^n$, let $\cD^n$ denotes the set of the corresponding zero-reward absorbing states, and $\bz^n$ denotes the corresponding vector $\bz$. We also define a matrix $\bZ$ of size $|\cS|\times N$ whose $n$-th column is vector $\bz^n$. Now, if we assume that the feature vector $f(s'|s)$ only depends on state $s'$ and $s$, then matrix $\bM$ are the same over $N$ trajectories. So, if we define a matrix $\bB (|\cS|\times N)$ with elements 
 $\bB_{s,n} = 1$ if $s\in \cD^n$, and   
$\bB_{s,n} = 0$ otherwise, 
then $\bZ$ and and the gradients of $\bZ$ w.r.t. a parameters $\theta_t$ are solutions to the following  systems of linear equations
\begin{equation}
\label{eq:grad-bigZ-compute}
    \bZ = (\bI-\bM)^{-1}\bB,
\text{ and }    
    \frac{\partial \bZ}{\partial \theta_t} = (\bI-\bM)^{-1}(\bM \circ \bF^t) \bZ. 
\end{equation}
In summary, to compute the  log-likelihood (and its gradients) of the demonstrated trajectories, one just need to solve  \eqref{eq:grad-bigZ-compute} to obtain $\bZ$ and  the Jacobian matrix of $\bZ$. This  requires to solve $(T+1)$ systems of linear equations ($T$ is the number of features considered). In large-scale applications, we typically do not invert the matrix $\bI-\bM$ to solve the linear systems. Instead, we can use more stable and less time-consuming methods (e.g.,\textit{ $\bL\bU$ factorization}). Since the linear systems  \eqref{eq:grad-bigZ-compute} all involve $(\bI-\bM)$, the \textit{$\bL\bU$ factorization} appears to be an efficient approach. The idea is that we decompose matrix $\bI - \bM$ into two matrices $\bL$ and $\bU$ as $\bI-\bM = \bL\bU$, where $\bL$ is a $(|\cS|\times |\cS|)$ lower triangular matrix with unit diagonal and $\bU$ is a $(|\cS|\times |\cS|)$ upper triangular matrix. Then, for solving a linear system $(\bI-\bM)\bX = \bc$, we just need to perform two simple tasks, namely, (i) finding  $\bY$ such that $\bL \bY = \bc$ and (ii) finding $\bX$ such that $\bU\bX = \bY$. Since the factorization step (i.e.,  finding $\bL,\bU$) is more expensive than the two steps (i) and (ii), we just need to do the factorization once and reuse the matrices $\bL,\bU$ for all the $T+1$ systems of linear equations. 

\section{Training Algorithm  with Missing Data}\label{sec:composition_missing}
We need demonstrated trajectories  $\{(s_1,a_1),\ldots,(s_K,a_K)\}$, $(s_i,a_i) \in \cS\times\cA$, $i=1,\ldots,K$,  to train an IRL model. We now consider the case where there are some missing state-action pairs on the trajectories. The question is how to train the IRL model in such situations. 

\subsection{Traditional Approaches}
To make the presentation simple, we consider a trajectory $\sigma$  containing a pair of states $(s_i,s_{i+1})$ such that state-action observations between these two states are missing. 
 To train the model, we need to compute or approximate $P(s_{i+1}|s_i)$, i.e., the probability of reaching state $s_{i+1}$ from state $s_i$. 
\textit{Naive approaches} might be to ignore all the missing segments, or to enumerate all possible states and actions that allow to move from $s_i$ to $s_{i+1}$. 
However, enumerating all possible paths between two states is not tractable in practice. 

Another conventional approach is to use the {EM} method, which is a popular way to deal with missing variables and has been considered to train IRL models with missing data \citep{IRL_Shahryari2017,IRL_Bogert2016}.  The idea is to alternate performing an expectation (E) step, which creates a function for the expectation of the log-likelihood evaluated using the current estimate for the parameters and a maximization (M) step that maximizing the expectation function created by the E step. 
The key feature of the EM is to define a function as the expected value of the log-likelihood function of the parameters $\theta$, w.r.t. the current conditional distribution of unobserved trajectories  given the observed state-actions and the current estimate of the model parameters. To this end, we need to 
build the function $g(\theta|s_i,s_{i+1}) = \mathbb{E}_{\gamma|(s_i,s_{i+1}),\theta^{t}}[\ln P(\gamma|\theta)]$
for each incomplete pair of states $(s_i,s_{i+1})$,
where $\theta^t$ is the parameter estimates at iteration $t$ and $\gamma$ is a complete trajectory from $s_i$ to $s_{i+1}$ (inclusive). We can approximate the expectation by sampling over the distribution of $\gamma$. 

There are two main issues associated with the use of the EM method. First, the EM requires to sample trajectories between any incomplete pair, which would be expensive, especially when the number of such pairs is huge.  
Another issue is that the EM method requires to solve several maximum likelihood estimation problems until getting a fixed point solution. This is indeed an expensive procedure. 

To sample $\gamma$, a straightforward approach is to start from $s_i$ and sample an action $a\in \cA$ and a next state according to probability $P(a|s_i,\theta^t)$ and $p(s'|a,s_i)$. Keep doing that until we reach $s_{i+1}$, we can get one complete trajectory $\gamma$ between $s_i$ and $s_{i+1}$. Suppose $\gamma_1,\ldots,\gamma_H$ are $H$ trajectory samples, the expectation can be approximated as $g(\theta|s_i,s_{i+1}) \approx \frac{1}{H} \sum_{h=1}^H \ln P(\gamma_h|\theta)$.

However, if $\theta^t$ is different from the optimal value such that $p(s_{i+1}|s_i, \theta^t)$ is small, then it is very inefficient to sample a complete trajectory from $s_i$ to $s_{i+1}$ following the above procedure. In the experiment, we only enumerate a subset of trajectories for each missing segment by employing \textit{breadth first search} without marking visited states and with a limited depth $H$. It is able to sample all possible trajectories of length less than or equal to $H$. If BFS with depth $H$ cannot find any path connecting a missing segment, the missing segment is omitted from the training. This is to keep the training time reasonable. Suppose these trajectory samples are $\gamma_1, \dots, \gamma_H$, then the expectation can be approximated as $g(\theta|s_i,s_{i+1}) \approx \sum_{h=1}^H P(\gamma_h|(s_i,s_{i+1}), \theta^t) \log P(\gamma_h|\theta)$.
We name this method as EM-BFS-H.
The purpose is to empirically assess how sampling a subset of paths affects the EM algorithm.

\subsection{Likelihood of Incomplete Trajectories}
In the following we present a tractable way to compute the \textit{log-likelihood of incomplete trajectories without path/trajectory sampling}. To do so, let us consider an incomplete pair $(s_i,s_{i+1})$ such that state-action pairs between these two states are missing. 
To compute the log-likelihood of the trajectory that contain pair $(s_i,s_{i+1})$, we need to compute $P(s_{i+1}|s_i)$, the probability of reaching $s_{i+1}$ from $s_i$. To do this in a tractable way, we define $\pi(s)$ as the probability of reaching state $s\in\cS$ from state $s_i$. The values of $\pi(s)$, $\forall s\in \cS$, satisfy the following equations 
\begin{equation}\label{eq:recursive-pi}
    \pi(s) = 
    \begin{cases}
    \sum_{s'\in \cS}\sum_{a\in \cA}p(s|a,s')P(a|s')\pi(s') & \forall s\neq s_{i}\\
    1 & s=s_i.
    \end{cases}
\end{equation}
So, if we define $\bd$ as a vector of size $|\cS|$ with all zero elements except $\bd_{s_i} = 1$ and 
a matrix $\bQ$ of size $(|\cS|\times|\cS|)$ with elements
\[
\bQ^{s_i}_{s,s'} =  
\begin{cases}
\sum_{a\in \cA}p(s|a,s')P(a|s'),& \forall s,s'\in \cS,\  s\neq s_{i} \\
0 &\text{ if } s = s_i,
\end{cases}
\]
then we can write \eqref{eq:recursive-pi} as 
\begin{equation}\label{eq:pi-linear-system}
\bpi = \bQ^{s_i} \bpi + \bd, \text{ or }\bpi = (\bI-\bQ^{s_i})^{-1}\bpi,    
\end{equation}
where $\bpi$ is a vector of size $(|\cS|)$ with entries $\bpi_s = \pi(s)$, $\forall s\in\cS$, and $\bI$ is the identity matrix of size $(|\cS|\times|\cS|)$. Proposition \ref{theor:invertible-IQ} shows that $\bI-\bQ^{s_i}$ is invertible, which guarantees the existence and uniqueness of the solutions to \eqref{eq:pi-linear-system}.
\begin{proposition}\label{theor:invertible-IQ}
$\bI-\bQ^{\Bar{s}}$ is invertible, for any $\Bar{s}\in\cS$. 
\end{proposition}
\begin{proof}
For any $\Bar{s}\in \cS$, we define $\bP = (\bQ^{\Bar{s}})^\transpose$ and write 
\[
\sum_{s'\in \cS} \bP_{s,s'} =  
\begin{cases}
 \sum_{s'\in \cS} \sum_{a\in \cA}p(s'|a,s)P(a|s) = 1 & \text{ if } s\neq \Bar{s} \\
 0 & \text{ if }s= \Bar{s}.
\end{cases} 
\]
So, $\bP$ is a \textit{sub-stochastic} matrix (contains non negative entries and every row adds up to at most 1). Moreover, $\bP$ contains no  recurrent class. So, $(\bI-\bP)$ is invertible. Furthermore, $(\bI-\bP)^\transpose = \bI - \bQ^{\Bar{s}}$ is also invertible, which is the desired result.  
\end{proof}

So,  we can obtain $P(s_{i+1}|s_i)$ as 
$P(s_{i+1}|s_i) = \bpi_{s_{i+1}}$. 
We also need the gradients of $P(s_{i+1}|s_i)$ for the maximum likelihood estimation. This can be obtained through taking the Jacobian of $\bpi$ w.r.t. parameters $\theta$ as
$
\bJ^\pi_\theta = (\bI-\bQ^{s_i})^{-1}\bH,
$
where $\bH$ is of size $|\cS|\times T$ whose $t$-th column is vector
$({\partial \bQ^{s_i}})/({\partial \theta_t}) \bpi$, recall that $T$ is the size of $\theta$ - number of features considered in the IRL model.
So, in summary, we can compute the log-likelihood of the demonstrated trajectories with missing data as well we its gradients by solving several \textit{ systems of linear equations.}

\subsection{Composition Algorithm}
We show above that one can obtain the log-likelihood of  trajectories with missing data by  solving one system of linear equations for each incomplete pair $(s_i,s_{i+1})$. This would be time-consuming if the demonstrated trajectories contains a large number of such pairs. We show in the following  that it is possible to 
obtain the probabilities of all the incomplete pairs in  trajectories sharing the same set of zero-reward absorbing states by solving only one system of linear equations, instead of solving one linear system per each incomplete pair. 

Assume that we observe $K$ incomplete pairs $\{(u_1,v_1)$,..., $(u_K,v_K)\}$ in demonstrated trajectories that share the same set of zero-reward absorbing states, for $(u_i,v_i) \in \cS\times \cS$, $\forall i=1,\ldots,K$. We define a matrix $\bQ^0$ of size $(|\cS|+1)\times (|\cS|+1)$ with entries
\[
\begin{aligned}
\bQ^0_{s,s'} &= \sum_{a\in \cA}p(s|a,s')P(a|s'),\ \forall s,s'\in\cS\\
\bQ^0_{s,|\cS|+1} &= \bQ^0_{|\cS|+1,s} = 0,\  \forall s\in\cS,
\end{aligned}
\]
and a matrix $\bD$ of size $(|\cS|+1)\times K$ with entries
\[
\bD_{s,k} =  1  \text{ if } s=|\cS|+1 \text{ or } s = u_k, \text{ and }
\bD_{s,k} =  0  \text{ otherwise. }
\]
The following theorem indicates how to obtain all the probabilities $P(v_k|u_k)$, $k=1,\ldots,K$ by solving only one system of linear equations.
\begin{theorem}\label{theo:th2}
If $\pmb{\Pi}$ is a solution to the system of linear equations $(\bI-\bQ^0)\pmb{\Pi} = \bD$, then 
$
P(v_k|u_k)= \pmb{\Pi}_{v_k, k},\ \forall k=1,\ldots,K.
$
\end{theorem}
The proof of Theorem \ref{theo:th2} is given in the Appendix. As a result, the gradient of $P(v_k|u_k)$ w.r.t.  a parameter $\theta_t$ can also be obtained via solving the following linear system
\begin{equation}\label{eq:gradient-missing-prob}
    \frac{\partial \pmb{\Pi}}{\partial \theta_t} = (\bI-\bQ^0)^{-1}\frac{\partial \bQ^0}{\partial \theta_t}\pmb{\Pi}.
\end{equation}
Moreover,
similarly to the proof of Proposition \ref{theor:invertible-IQ}, we can show that $(\bQ^0)^\transpose$ is a \textit{sub-stochastic} matrix with no recurrent class. As a result, $\bI-(\bQ^0)^\transpose$ is invertible. So $\bI-(\bQ^0)$  is also invertible,
which ensures that the linear system in Theorem \ref{theo:th2} and \eqref{eq:gradient-missing-prob} always have unique solutions.

Algorithm \ref{algo:LL-missing} describes basic steps to compute the log-likelihood value and its gradients in the case of missing data. The total number of linear systems to be solved in Algorithm \ref{algo:LL-missing} is $(T+1)(1+N^{\textsc{dest}})$, where $N^{\textsc{dest}}$ stands for the number of groups of trajectories, where each group contains trajectories that share the same set of zero-reward absorbing states. Clearly, the  number of linear systems to be solved does not depend on the number of missing segments in the demonstrated trajectories.
Moreover,  Algorithm \ref{algo:LL-missing} can be implemented in a parallel manner, which would help to speed up the computation.  
If the data is complete (no missing segment), then we just need to remove Step 2. Note that in this case, we only need to solve $T+1$ systems of linear equations to obtain the log-likelihood value and its gradients, and this number does not depend on the number of demonstrated trajectories.

\begin{algorithm}[htb]
	\SetKwRepeat{Do}{do}{while}
	\comments{Log-likelihood computation of demonstrated trajectories with missing data} \\
	\Begin
	{
		1. Compute $\bZ$ and its gradients  $\partial \bZ /\partial \theta_t$ using  \eqref{eq:grad-bigZ-compute}. \\
		2. For each set of zero-reward absorbing states, compute the probabilities (and their gradients) of incomplete pairs by solving systems $(\bI - \bQ^0)\bPi = \bD$ and \eqref{eq:gradient-missing-prob}. \\
		3. Compute the log-likelihood of incomplete trajectories and its gradients.
	}
	\caption{(\textbf{\textit{Composition algorithm}})\label{algo:LL-missing}}
\end{algorithm}

Similarly to the previous section,  for each set of zero-reward absorbing states,  Step 2 requires to solve $T+1$ systems of linear equations that all involve the matrix $\bI-\bQ^0$. The $\bL\bU$ factorization is also a convenient approach to achieve good performance. Technical speaking, one can firstly decompose $\bI-\bQ^0$ into a lower triangular matrix $\bL$ and an upper  triangular matrix $\bU$ and use these two matrices to solve the $T+1$ linear systems in Step 2.

\section{Experiments}\label{sec:experiments}
This section evaluates the empirical performance of our approach using a real-world dataset.
We compare our decomposition method with the EM method and the (naive) one that simply ignores the missing segments in the dataset.
 Our dataset contains larger numbers of states/actions/trajectories, as compared to those used with the EM  in previous studies \citep{IRL_Shahryari2017,IRL_Bogert2016}, which makes it more suitable to illustrate the scalability of our approach in handling the issue of missing data.
For the sake of comparison, we also compare our approach with the traditional \textit{value iteration} \citep{Ziebart2008maximum} in the case of no-missing data.

{\bf{Dataset.}} The dataset contains 1832 real-world trajectories of taxi drivers.
The road network consists of 7288 links, which are states in the model.
The action is moving from a link to a consecutive link with no uncertainty. There are four features, in which three boolean features are left turn, U-turn, and the incidence matrix, and one real-valued feature is the traveling time between consecutive links. This dataset has been used in some  route choice modeling studies \citep{Fosgerau2013RL,Mai2015Nested}.  

{\bf{Generating Missing Dataset.}} 
To evaluate the performances of different approaches in the context of missing data, we take the full trajectories from the above dataset and remove some observed links to generate missing datasets. This allows us to assess our algorithm with different ``\textit{missing levels}''. We define
\emph{missing probability}  as the probability that a link (except the origin and the destination links) is removed from a trajectory in the dataset. To increase the difficulty of recovering the reward function, the removed links in a trajectory are consecutive. Taking a trajectory of length $l$ as an example, the length of the missing segment follows a binomial distribution with $l-2$ trials, and the missing probability as the success probability. For each trajectory, such a length is randomly sampled, and a segment of this length is randomly removed from the trajectory. In the experiment, there are 10 different random runs to account for the randomness in the generation of the missing dataset.

In this setting, as the missing probability increases, the expected length of missing segments in the dataset increases. This, in turn, makes the learning of the reward function more difficult as observed in the experimental result.
To assess the performance of the reward recovery, the log-likelihood of the no-missing dataset using the learned reward function is computed. The higher the log-likelihood, the better the reward function is.

{\bf{Comparison between Composition and MaxEnt Algorithms on No-Missing Dataset.}} To compare the composition algorithm with the popular MaxEnt IRL algorithm (i.e., \textit{value iteration}) \citep{Ziebart2008maximum}, we run both algorithms on the no-missing dataset. They both give similar reward functions, i.e., $(-0.993, -0.999, -21.413, -1.973)$ (MaxEnt) and $(-0.995, -1.019, -20.248, -1.971)$ (composition), which have no significant difference in the log-likelihood ($-2646.146$ and $-2646.004$, respectively). On the other hand, the iterative algorithm to compute the log-likelihood (and its derivative) of MaxEnt incurs much more time ($96.1$s) than that of the composition algorithm ($1.50$s). So, our composition algorithm is about 60 times faster than the MaxEnt while returning a similar reward function.  

{\bf{Comparison between Composition and EM Algorithms on Missing Dataset.}}
Fig.~\ref{fig:comparison-results} shows a comparison of different  approaches in terms of the reward recovery and computing time. 
Regarding the performance, Fig.~\ref{fig:loglikelihood} shows the log-likelihood of no-missing dataset with the learned rewards from different methods.
When the missing probability is less than $0.4$, the information from missing segments is not enough to make a visible difference among all the methods. However, as the missing probability increases to $0.9$, the use of missing segments makes a significant rise in the performance of the composition method over the others: the composition method's log-likelihood achieves both a larger value and a smaller $95\%$ confidence interval.
Surprisingly, EM methods are outperformed by the naive method using only connected segments when the missing probability is large. It is because the length of missing segments increases in this case, which makes the sampling of BFS-$5$ and BFS-$8$ inaccurate. This illustrates the adverse effect of sampling a subset of trajectories when the missing segment length is large. On the other hand, increasing the depth of BFS increases the execution time significantly (Fig.~\ref{fig:exectime}). Regarding our composition method, its execution time to compute the log-likelihood is remarkably smaller than that of the EM\footnote{The execution time is measured using a computer with: i7-6700 CPU @ 3.40GHz and 16GB of RAM.}. It is also interesting to see that, as expected, the computing time of the composition method does not grow much as the missing probability increases. These observations demonstrate the scalability of the composition method in a real-world application.







\vspace{-0.9em}
\begin{figure}[htb]
    \centering
    	\begin{subfigure}{.48\textwidth}
    	    \centering            \includegraphics[width = \linewidth]{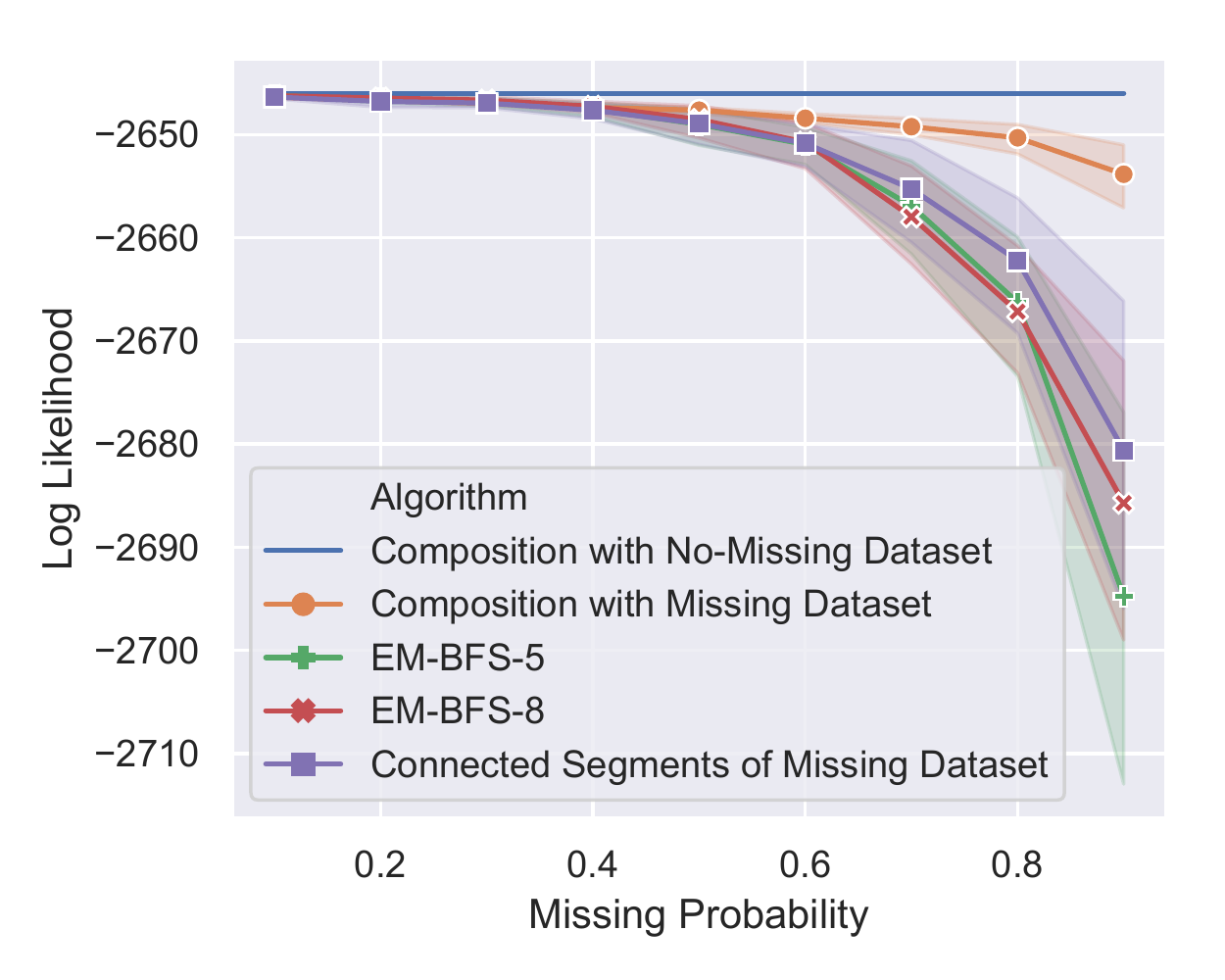}
    	    \caption{Log-likelihood and its $95\%$ confidence interval.}
            \label{fig:loglikelihood}
    	\end{subfigure} \ 
    	\begin{subfigure}{.48\textwidth}
    	    \centering
            \includegraphics[width = \linewidth]{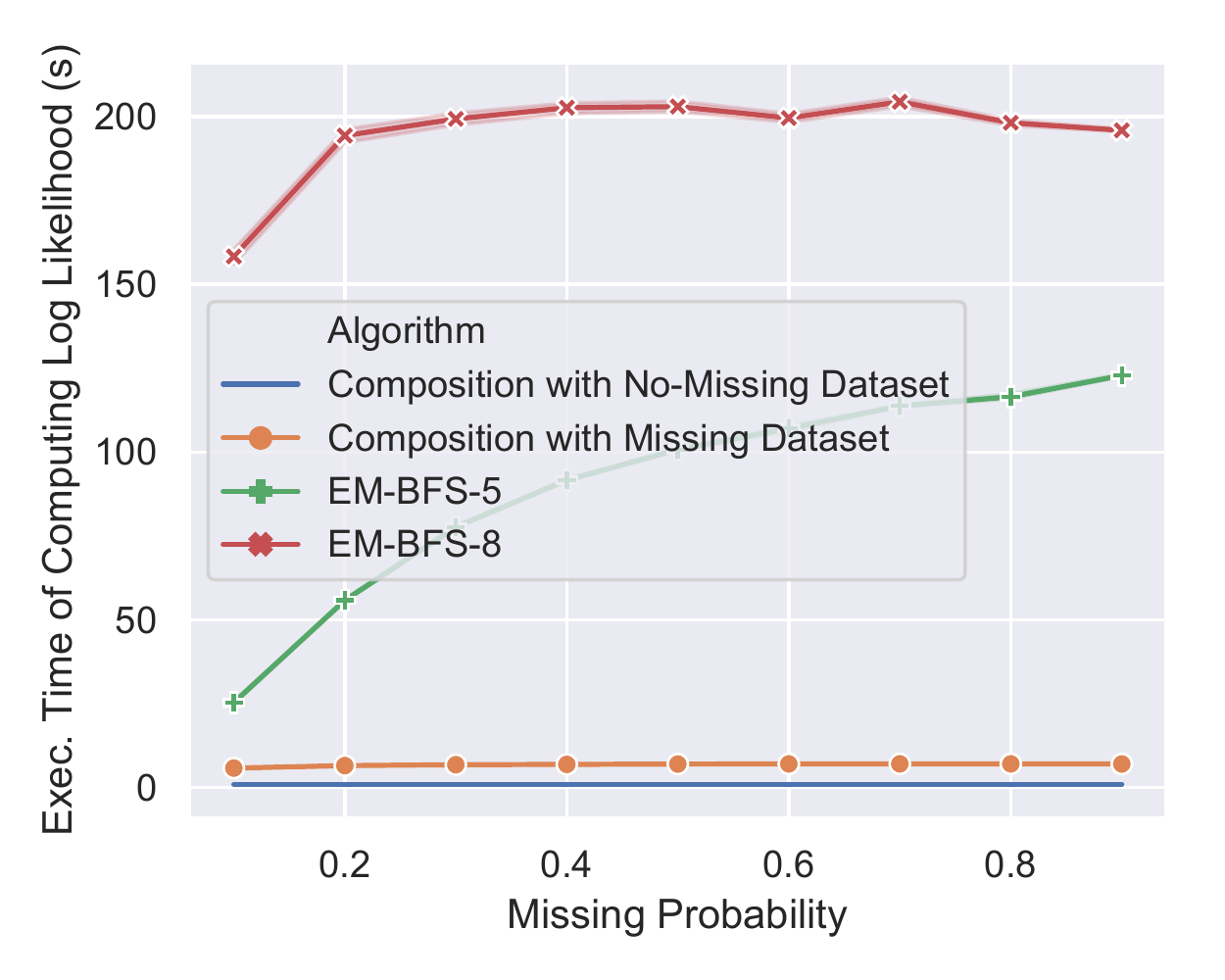}
            \caption{Execution time of computing the log-likelihood.}
            \label{fig:exectime}
    	\end{subfigure} 
    \caption{Comparison results between different methods: \emph{Composition with Missing Dataset} denotes the composition method on missing dataset, \emph{EM-BFS-$x$} denotes the EM algorithm with BFS of depth $x$, and \emph{Connected Segments of Missing Dataset} denotes the composition method trained with only the connected segments in the missing dataset. The log-likelihood of composition method on no-missing dataset is also plotted for reference, denoted as \emph{Composition with No-Missing Dataset}.}
    \label{fig:comparison-results}
\end{figure}


\section{Conclusion}\label{sec:concl}
We have presented a novel IRL training algorithm  that clearly outperforms previous approaches, in both cases of no-missing or missing data. The main advantage of our approach is that the IRL training procedure is performed via solving a number of systems of linear equations, and this number does not depend on the number of demonstrated trajectories in the case of complete data, and not depend on the number of missing segments in the missing case. 
This makes our approach highly scalable for large-scale applications.
Many applications would potentially benefit from our approach, for example, problems of learning expert's reward functions when datasets are not entirely available due to technical issues or privacy concerns. 
In future work, we plan to investigate generative adversarial ideas \citep{ho2016generative,goodfellow2014generative} for imitation learning in the context of missing information. 


\bibliographystyle{plainnat_custom}
\bibliography{refsIRL}

\begin{thebibliography}{16}
\providecommand{\natexlab}[1]{#1}
\providecommand{\url}[1]{\texttt{#1}}
\expandafter\ifx\csname urlstyle\endcsname\relax
  \providecommand{\doi}[1]{doi: #1}\else
  \providecommand{\doi}{doi: \begingroup \urlstyle{rm}\Url}\fi

\bibitem[Abbeel and Ng(2004)]{Ng04}
Abbeel, P. and Ng, A.~Y.
\newblock Apprenticeship learning via inverse reinforcement learning.
\newblock In \emph{Proc. {ICML}}, 2004.

\bibitem[Bogert and Doshi(2014)]{bogert2014multi}
Bogert, K. and Doshi, P.
\newblock Multi-robot inverse reinforcement learning under occlusion with
  interactions.
\newblock In \emph{Proc. AAMAS}, pages 173--180, 2014.

\bibitem[Bogert and Doshi(2017)]{bogert2017scaling}
Bogert, K. and Doshi, P.
\newblock Scaling expectation-maximization for inverse reinforcement learning
  to multiple robots under occlusion.
\newblock In \emph{Proceedings of the 16th Conference on Autonomous Agents and
  MultiAgent Systems}, pages 522--529. International Foundation for Autonomous
  Agents and Multiagent Systems, 2017.

\bibitem[Bogert et~al.(2016)Bogert, Lin, Doshi, and Kulic]{IRL_Bogert2016}
Bogert, K., Lin, J. F.-S., Doshi, P., and Kulic, D.
\newblock Expectation-maximization for inverse reinforcement learning with
  hidden data.
\newblock In \emph{Proceedings of the 2016 International Conference on
  Autonomous Agents \& Multiagent Systems}, pages 1034--1042. International
  Foundation for Autonomous Agents and Multiagent Systems, 2016.

\bibitem[Fosgerau et~al.(2013)Fosgerau, Frejinger, and
  Karlstrom]{Fosgerau2013RL}
Fosgerau, M., Frejinger, E., and Karlstrom, A.
\newblock A link based network route choice model with unrestricted choice set.
\newblock \emph{Transportation Research Part B: Methodological}, 56:\penalty0
  70--80, 2013.

\bibitem[Goodfellow et~al.(2014)Goodfellow, Pouget-Abadie, Mirza, Xu,
  Warde-Farley, Ozair, Courville, and Bengio]{goodfellow2014generative}
Goodfellow, I., Pouget-Abadie, J., Mirza, M., Xu, B., Warde-Farley, D., Ozair,
  S., Courville, A., and Bengio, Y.
\newblock Generative adversarial nets.
\newblock In \emph{Advances in neural information processing systems}, pages
  2672--2680, 2014.

\bibitem[Granas and Dugundji(2013)]{Granas2013fixed}
Granas, A. and Dugundji, J.
\newblock \emph{Fixed point theory}.
\newblock Springer Science \& Business Media, 2013.

\bibitem[Ho and Ermon(2016)]{ho2016generative}
Ho, J. and Ermon, S.
\newblock Generative adversarial imitation learning.
\newblock In \emph{Advances in Neural Information Processing Systems}, pages
  4565--4573, 2016.

\bibitem[Levine et~al.(2011)Levine, Popovi{\'{c}}, and Koltun]{Levine11}
Levine, S., Popovi{\'{c}}, Z., and Koltun, V.
\newblock Nonlinear inverse reinforcement learning with {Gaussian} processes.
\newblock In \emph{Proc. {NIPS}}, pages 19--27, 2011.

\bibitem[Mai et~al.(2015)Mai, Fosgerau, and Frejinger]{Mai2015Nested}
Mai, T., Fosgerau, M., and Frejinger, E.
\newblock A nested recursive logit model for route choice analysis.
\newblock \emph{Transportation Research Part B: Methodological}, 75:\penalty0
  100--112, 2015.

\bibitem[Ng and Russell(2000)]{Ng00}
Ng, A.~Y. and Russell, S.
\newblock Algorithms for inverse reinforcement learning.
\newblock In \emph{Proc. {ICML}}, pages 663--670, 2000.

\bibitem[Penny and Atkinson(2012)]{Penny2012approaches}
Penny, K.~I. and Atkinson, I.
\newblock Approaches for dealing with missing data in health care studies.
\newblock \emph{Journal of clinical nursing}, 21\penalty0 (19pt20):\penalty0
  2722--2729, 2012.

\bibitem[Russell(1998)]{Russell98}
Russell, S.
\newblock Learning agents for uncertain environments.
\newblock In \emph{Proc. {COLT}}, pages 101--103, 1998.

\bibitem[Shahryari and Doshi(2017)]{IRL_Shahryari2017}
Shahryari, S. and Doshi, P.
\newblock Inverse reinforcement learning under noisy observations.
\newblock \emph{arXiv preprint arXiv:1710.10116}, 2017.

\bibitem[Wang et~al.(2002)Wang, Rosenfeld, Zhao, and
  Schuurmans]{wang2002latent}
Wang, S., Rosenfeld, R., Zhao, Y., and Schuurmans, D.
\newblock The latent maximum entropy principle.
\newblock In \emph{Proceedings IEEE International Symposium on Information
  Theory,}, page 131. IEEE, 2002.

\bibitem[Ziebart et~al.(2008)Ziebart, Maas, Bagnell, and
  Dey]{Ziebart2008maximum}
Ziebart, B.~D., Maas, A.~L., Bagnell, J.~A., and Dey, A.~K.
\newblock Maximum entropy inverse reinforcement learning.
\newblock In \emph{AAAI}, volume~8, pages 1433--1438. Chicago, IL, USA, 2008.

\end{thebibliography}

\newpage

\section*{Appendix}

\subsection*{Proof of Theorem \ref{theor:theor-1}}
\text{We first prove the convergence  under Condition (i)} by introducing the following lemma
\begin{lemma}\label{lemma:theo1-lm1}
For any sequence $(s_1,\ldots,s_K)$ with $s_j\in \cS$, $j=1,\ldots,K$, and $K>|\cS|$, there is at least one pair $(s_i,s_{i+1})$ such that
\[
p(s_{i+1}|a,s_i) = 0,\ \forall a\in \cA.
\]
\end{lemma}
\begin{proof}
Because $K> |\cS|$ and  there are $K$ states in  $(s_1,\ldots,s_K)$, there are two identical states in the sequence, i.e., $\exists h,k\in\{1,\ldots, K\}$ such that $h<k$ and  $s_h = s_k$. The hypothesis of (i) guarantees that there exists $i$, $h\leq i<k$, such that 
\[
\sum_{a\in\cA}p(s_{i+1}|a,s_i) = 0,
\]
which leads to the desired result.
\end{proof}

Now we consider, for any $n\in \mathbb{N}$, matrix $\bM^n = \underbrace{\bM\times \bM\times \ldots\times \bM}_{\text{$n$ times}}$ with entries 
\[
\begin{aligned}
\bM^n_{s,s'} &= \sum_{\substack{s_1,\ldots,s_n\in\cS\\s_1  =s, s_n=s'}} \left(\prod_{i=1}^{n-1}\bM_{i,i+1}\right)\\
 &= \sum_{\substack{s_1,\ldots,s_n\in\cS\\s_1  =s, s_n=s'}} \prod_{i=1}^{n-1} \left(\sum_{a\in\cA} p(s_{i+1}|a,s_i)e^{r(s_{i+1}|s_i,\theta)} \right)
\end{aligned}
\]
Using Lemma \ref{lemma:theo1-lm1} we have that, if $n>|\cS|$, then 
\[
\prod_{i=1}^{n-1} \left(\sum_{a\in\cA} p(s_{i+1}|a,s_i)e^{r(s_{i+1}|s_i,\theta)} \right) = 0
\]
for any sequence $\{s_1,\ldots,s_n\}\in \cS^n$. So, $\bM^n_{s,s'}=0$ for any $s,s'\in\cS$ and $n>|\cS|$. In other words, 
\[
\lim_{n\rightarrow \infty }\bM^n= 0.
\]
This implies that the modulus of the eigenvalues of $\bM$ lie within the unit disc, and consequently, $(\bI-\bM)$ is invertible.

To prove the convergence under Condition (ii), we write matrix $\bM$  in the following canonical form
\[
\bM= 
\left(\begin{array}{c|c}
\pmb{\Sigma} & \textbf{R} \\ 
\hline
\textbf{O}_1 & \textbf{O}_2
\end{array}\right),
\]
where 
$\left(\begin{array}{c|c}
\textbf{O}_1 & \textbf{O}_1
\end{array}\right)$ and $\left(\begin{array}{c}
\textbf{R} \\ 
\hline
 \textbf{O}_2
\end{array}\right)$  
are the last rows and last columns of $\bM$ corresponding  the absorbing states $s\in\cD$, and $\textbf{O}_1$ and $\textbf{O}_2$ are matrices of zero elements. 
We define a matrix $\bM'$ of the same size with $\bM$ as
\[
\bM' = 
\left(\begin{array}{c|c}
	\pmb{\Sigma} & \textbf{R}' \\ 
	\hline
	\textbf{O}_1 & \bI
\end{array}\right),
\]
where $\bI$ is the identity matrix of size $(|\cD|\times|\cD|)$ and
$\textbf{{R}}'$ is a matrix of size $(|\cS|-|\cD|)\times |\cD|$ with entries
\[
\textbf{R}'_{s,k} = \frac{1}{|\cD|}\left(1-\sum_{s'\in\cS}\bM_{s,s'}\right)>0,\;\forall s\in\cS\backslash \cD, \forall k\in\cD
\]
So, under the hypothesis of the theorem, it is easy to verify that $\bM'$ is a transition matrix of an arbitrary absorbing Markov chain. Thus, we have the following well-known result
\begin{equation*}
\lim_{n \rightarrow \infty} \pmb{\Sigma}^n = 0.
\end{equation*}
So
\[\lim_{n\rightarrow \infty }\bM^n=\lim_{n\rightarrow \infty } 
\left(\begin{array}{c|c}
\pmb{\Sigma}^n & \pmb{\Sigma}^{n-1}\textbf{R} \\ 
\hline
\textbf{O}_1 & \textbf{O}_2
\end{array}\right) = 0.
\]
Similarly to the proof of the case (i), $(\bI-\bM)$ is invertible. 

\subsection*{Proof of Proposition \ref{prop:convergence-value-iteration}}

First, we assume that Condition (i) holds. At iteration $k>0$ of the value iteration procedure, we write
$
\bz^k =\bM^{k}\bz^0+ \sum_{j=0}^{k-1} \bM^{j} \bb, 
$
where $\bM^k$ stands for $\bM\times\ldots\times \bM$ ($k$ times).
Using the proof of Theorem \ref{theor:theor-1}, we have $\bM^k = 0$ 
for any $k>|\cS|$. So, we can write 
$
\bz^k = \sum_{j=0}^{|\cS|}\bM^j\bb,     
$
which does not depend on $\bz^0$ and $k$. This remark indicates that $\bz^k$ converges to the unique fixed point solution $\sum_{j=0}^{|\cS|}\bM^j\bb$ after $|\cS|$-th iterations. This completes the proof for Condition (i).

We now move to Condition (ii).
We define $f(\bz) = \bM \bz + \bb$. Assume that Condition (i) of Theorem \ref{theor:theor-1} holds, given a  vector $\bz\in[0,1]^{|\cS|}$, we have 
$
f(\bz)_s = \sum_{s'\in \cS}\bM_{s,s'}\bz_{s'} + \bb_s \leq \sum_{s'\in \cS}\bM_{s,s'}\max_{s'\in S}\{\bz_{s'}\} + \bb_s  \leq  1.
$
So, $f(\bz)$ is a mapping from $[0,1]^{|\cS|}$ to itself. According to the Brouwer fixed-point theorem \citep{Granas2013fixed}, there is at least one fixed point solution in $[0,1]^{|\cS|}$. As stated in Theorem  \ref{theor:theor-1}, there always exists a unique fixed point solution under Condition (i), meaning that the fixed point solution $\bz^*$ always lies in $[0,1]^{|\cS|}$. Moreover, given two points $\bz^1,\bz^2 \in [0,1]^{|\cS|}$ we have
\[
\begin{aligned}
||f(\bz^1) - f(\bz^2)||_{\infty} &=\max_{s\in \cS}\left|\sum_{s'\in\cS} \bM_{s,s'} \bz^1_{s'} - \bM_{s,s'} \bz^2_{s'}\right|\\
&\leq \max_{s\in \cS}\left(\sum_{s'\in\cS}\bM_{s,s'}\right) \max_{s}\left|\bz^1_s-\bz^2_s\right| \\
&\leq \tau ||\bz^1-\bz^2||_\infty,
\end{aligned}
\]
where $\tau = \max_{s}\left(\sum_{s'\in\cS}\bM_{s,s'}\right)<1$. So, $f(\bz)$ is a contraction mapping in $[0,1]^{|\cS|}$, meaning that the value iteration always converges to a fixed point solution from any starting point. As shown previously, this fixed point solution is unique and lies in $[0,1]^{|\cS|}$. We complete the proof for Condition (i).

\subsection*{Proof for Theorem \ref{theo:th2}}
First, let us consider a pair of missing data $(u_k,v_k)$. We create an artificial state $h$ such that
\[
\begin{cases}
 \sum_{a\in \cA}p(h|a,s) = 0 &\forall s\in \cS\\
 \sum_{a\in \cA}p(s|a,h) = 0 & \forall s\in \cS\backslash \{u_k\}\\
 \exists a\in \cA\text{ such that } p(u_k|a,h) = 1.
\end{cases}
\]
Basically, we require that it is impossible to reach $h$ from other states in $\cS$ and  from $h$ we can  only reach $v_k$ with probability 1. Now we define the  following recursive equations
\begin{equation}\label{eq:recursive-pi'}
    \pi^k(s) = 
    \begin{cases}
    \sum_{s'\in \cS}\sum_{a\in \cA}p(s|a,s')P(a|s')\pi^k(s') & \forall s\in \cS\\
    1 & s=h.
    \end{cases}
\end{equation}
One can show that if function $\pi^k:\cS\cup\{h\}\rightarrow [0,1]$ satisfies \eqref{eq:recursive-pi'}, then $\pi^k(s) = P(s|u_k)$, $\forall s\in\cS$. 
Now, we define a matrix $\bQ^k$ of size $(|\cS|+1) \times (|\cS|+1)$ with elements
$\bQ^k_{s,s'} =  \sum_{a\in \cA}p(s|a,s')P(a|s')$, $\forall s,s'\in \cS\cup\{h\}$
and a vector $\bq$ of size $|\cS|+1$ with all zero elements except $\bq_h = 1$. We also number state $h$ as $|\cS|+1$, so the last column and row of $\bQ^k$ and the last element of $\bq$ correspond to state $h$. So, \eqref{eq:recursive-pi'} can be written as 
\begin{equation}\label{eq:linear-system-pik}
\bpi^k = \bQ^k \bpi^k+\bq.    
\end{equation}
Note that the last row of $\bQ^k$ is an all-zero vector. 
We  decompose \eqref{eq:linear-system-pik} as
\[\begin{aligned}
\bQ^k\bpi^k+\bq  &=
\begin{pmatrix}
\bQ^k_{1,1} & \bQ^k_{1,2} & \cdots & \bQ^k_{1,|\cS|+1} \\
\bQ^k_{2,1} & \bQ^k_{2,2} & \cdots & \bQ^k_{2,|\cS|+1} \\
\vdots  & \vdots  & \ddots & \vdots  \\
0 & 0 & \cdots & 0
\end{pmatrix}
\begin{pmatrix}
\bpi^k_1\\
\bpi^k_2\\
\vdots\\
\bpi^k_{|\cS|+1}
\end{pmatrix} +\bq \\
&=\bQ^0 \bpi^k+\begin{pmatrix}
0 & 0 & \cdots & \bQ^k_{1,|\cS|+1} \\
0 & 0 & \cdots & \bQ^k_{2,|\cS|+1} \\
\vdots  & \vdots  & \ddots & \vdots  \\
0 & 0 & \cdots & 0
\end{pmatrix}
\begin{pmatrix}
\bpi^k_1\\
\bpi^k_2\\
\vdots\\
\bpi^k_{|\cS|+1}
\end{pmatrix} +\bq \\
&= \bQ^0 \bpi^k+ \begin{pmatrix}
\bQ^k_{1,|\cS|+1}\\
\bQ^k_{2,|\cS|+1}\\
\vdots\\
0
\end{pmatrix} \bpi^k_{|\cS|+1} + \bq = \bQ^0\bpi^k + \bD_{:,k}
\end{aligned}
\] 
where $\bD_{:,k}$ is the $k$-th column of $\bD$. So, $\bpi^k$ is a solution to the following system
\[
(\bI-\bQ^0)\bpi^k = \bD_{:,k},
\]
which also means that if matrix $\pmb{\Pi}$ is a solution to the system of linear equations $(\bI-\bQ^0)\pmb{\Pi} = \bD$, then
\[
\bpi^k = \pmb{\Pi}_{:,k}\text{ and } \pmb{\Pi}_{v_k,k} = \bpi^k(v_k) = P(v_k|u_k),
\]
which is the desired result.

\end{document}